\documentclass{article}

\usepackage[accepted]{icml2021}

\usepackage{microtype}
\usepackage{subcaption}
\usepackage{algorithm}
\usepackage{algorithmic}
\usepackage{url}
\usepackage{graphicx}
\usepackage{hyperref}
\usepackage{listings}
\usepackage{booktabs}
\usepackage{amsmath}
\usepackage{amssymb}
\usepackage{amsthm}
\usepackage{xcolor}
\usepackage{dsfont}
\usepackage{paralist}

\newtheorem{dfn}{Definition}
\newtheorem{thm}{Theorem}
\newtheorem{pro}{Proposition}

\allowdisplaybreaks

\newcommand{\stochsim}{\ensuremath{\sim}}
\newcommand{\stoch}{\ensuremath{{\sim}}}
\newcommand{\soft}{\ensuremath{{\overset{\text{\tiny{soft}}}{\sim}}}}

\newcommand{\var}[1]{{\operatorname{\mathit{#1}}}}

\newcommand{\appropto}{\mathrel{\vcenter{
  \offinterlineskip\halign{\hfil$##$\cr
    \propto\cr\noalign{\kern1pt}\sim\cr\noalign{\kern-1pt}}}}}

\lstdefinelanguage{Golang}%
  {morekeywords=[1]{package,import,func,type,struct,return,defer,panic,%
     recover,select,var,const,iota,},%
   morekeywords=[2]{string,uint,uint8,uint16,uint32,uint64,int,int8,int16,%
     int32,int64,bool,float,float32,float,complex64,complex128,byte,rune,uintptr,%
     error,interface},%
   morekeywords=[3]{map,slice,make,new,nil,len,cap,copy,close,true,false,%
     delete,append,real,imag,complex,chan,},%
   morekeywords=[4]{for,break,continue,range,go,goto,switch,case,fallthrough,if,%
     else,default,},%
   morekeywords=[5]{Println,Printf,Error,Print,},%
   sensitive=true,%
   morecomment=[l]{//},%
   morecomment=[s]{p*}{*/},%
   morestring=[b]',%
   morestring=[b]",%
   morestring=[s]{`}{`},%
}

\begin{document}

\twocolumn[

\icmltitle{Probabilistic Programs with Stochastic Conditioning}

\begin{icmlauthorlist}
\icmlauthor{David Tolpin}{bgu}
\icmlauthor{Yuan Zhou}{ox}
\icmlauthor{Tom Rainforth}{ox}
\icmlauthor{Hongseok Yang}{kaist}
\end{icmlauthorlist}

\icmlaffiliation{bgu}{Ben-Gurion University of the Negev}
\icmlaffiliation{ox}{University of Oxford}
\icmlaffiliation{kaist}{School of Computing, KAIST}

\icmlcorrespondingauthor{David Tolpin}{david.tolpin@gmail.com}

\icmlkeywords{probabilistic programming, stochastic conditioning}

\vskip 0.3in

]

\printAffiliationsAndNotice{}  

\begin{abstract}
	We tackle the problem of conditioning probabilistic
	programs on distributions of observable variables.
	Probabilistic programs are usually conditioned on samples
	from the joint data distribution, which we refer to as deterministic conditioning. 
	However, in many real-life
	scenarios, the observations are given as marginal
	distributions, summary statistics, or samplers.
	Conventional probabilistic programming systems lack adequate
	means for modeling and inference in such scenarios. 
	We propose a generalization of deterministic conditioning 
	to \emph{stochastic conditioning}, 
	that is,  
	conditioning on the marginal distribution of a
	variable taking a particular form. 
	To this end, we first define the formal notion of stochastic conditioning 
	and discuss its key properties. 
	We then show how to perform inference in the presence of stochastic conditioning.
	We demonstrate potential usage of stochastic conditioning on 
	several case studies which involve various 
	kinds of stochastic conditioning and are difficult to
	solve otherwise. 
	Although we present stochastic
	conditioning in the context of probabilistic programming,
	our formalization is general and applicable to other settings. 
\end{abstract}

\section{Introduction}

Probabilistic programs implement statistical models. Mostly,
probabilistic programming closely follows the Bayesian
approach~\citep{KP83,GCS+13}: a prior distribution is imposed on latent
random variables, and the posterior distribution is conditioned
on observations (data). The conditioning may take the form of a
hard constraint (a random variate must have a particular value, 
e.g. as in Church~\citep{GMR+08}), but a more common practice is
that conditioning is soft --- the observation is assumed to come
from a distribution, and the probability (mass or density) of
the observation given the distribution is used in
inference~\citep{Stan17,GS15,TMY+16,GXG18}. In the standard setting,
the conditioning is \textit{deterministic} --- observations are fixed
samples from the joint data distribution.  For example, in the
model of an intensive care unit patient, an observation may be a
vector of vital sign readings at a given time.  This setting is
well-researched, widely applicable, and robust inference is
possible~\citep{HG11,WVM14,TFP18,BCJ+19}.

However, instead of samples from the joint data distribution,
observations may be independent samples from marginal data
distributions of observable variables, summary statistics,
or even data distributions themselves, provided in closed form or as
samplers. These cases naturally appear in real life scenarios:
samples from marginal distributions arise when different
observations are collected by different parties, summary
statistics are often used to represent data
about a large population, and data distributions may express
uncertainty during inference about future states of the world,
e.g. in planning. Consider the following situations:
\begin{compactitem}
	\item A study is performed in a hospital on a group of
		patients carrying a certain disease. To preserve the
		patients' privacy, the information is collected and
		presented as summary statistics of each of the monitored
		symptoms, such that only marginal distributions of the
		symptoms are approximately characterized. It would be
		natural to condition the model on a combination of
		symptoms, but such combinations are not observable.
	\item A traveller regularly drives between two cities and wants to
		minimize the time this takes. However, some road sections may be closed 
		due to bad weather, which can
		only be discovered at a crossing adjacent to the road  section. 
		A policy that minimizes average travel time,
		given the probabilities of each road closure,
		is required, but finding this policy requires us to condition on the
		\textit{distribution} of states~\citep{PY89}.
\end{compactitem}

Most existing probabilistic programming systems, which condition on
samples from the joint data distribution, cannot be directly applied to
such scenarios for either model specification or performing inference.  
In principle, such models can be expressed as nested
probabilistic programs~\citep{R18}. 
However, inference in such
programs has limited choice of algorithms, is computationally
expensive, and is difficult to implement correctly~\citep{RCY+18}.

In some specific settings, models can be augmented with additional
auxiliary information, and custom inference techniques can be used. 
For example, \citet{MPT+16} employs black-box variational inference
on augmented probabilistic programs for policy search.
But problem-specific
program augmentation and custom inference compromise the core
promise of probabilistic programming: programs and algorithms
should be separated, and off-the-shelf inference methods 
should be applicable to a wide range of programs in an
automated, black-box manner~\citep{WVM14,T19,WBD+19}. 

To address these issues and provide a general solution to defining models and running inference
for such problems, we propose a way to extend
\textit{deterministic conditioning} $p(x\vert y=y_0)$, i.e.~conditioning on some random
variable in our program $y$ taking on a particular value $y_0$,
to \textit{stochastic conditioning} $p(x\vert y \stoch D_0)$, i.e.
conditioning on $y$ having the marginal distribution $D_0$.
In the context of a higher-order sampling process in which we first sample a random probability measure $\mathbf{D}\sim p(\mathbf{D})$ and then sample a random variable $y \sim \mathbf{D}$,  stochastic conditioning $p(x\vert y \stoch D_0)$ amounts to conditioning on the event $\mathbf{D}=D_0$,
that is on the random measure $\mathbf{D}$ itself taking the particular form $D_0$.
Equivalently, we can think on conditioning on the event $y \sim D_0$ (based on the first step of our sampling process), which says that the marginal distribution of $y$ is given by the distribution~$D_0$. 
We can develop intuition for this by considering the special case of a discrete $y$, where
$y \sim D_0$ means that the proportion of each possible instance of $y$ that occurs will be $D_0$ if we conduct an infinite number of rollouts and sample a value of $y$ for each.  

To realize this intuition, we formalize stochastic conditioning and analyze its properties and usage in the context of probabilistic programming,
further showing how effective automated inference engines can be set up
for the resulting models.
We note that our results also address a basic conceptual problem in Bayesian modeling,
 and are thus applicable to non-probabilistic programming settings as well. 

We start with an informal introduction providing
intuition about stochastic conditioning
(Section~\ref{sec:intuition}).  Then, we define the notion of
stochastic conditioning formally and discuss its key
properties (Section~\ref{sec:stochastic-conditioning}), 
comparing our definition with possible alternatives
and related concepts.  Following that, we discuss efficient
inference for programs with stochastic conditioning
(Section~\ref{sec:inference}).  In case studies
(Section~\ref{sec:case-studies}), we provide probabilistic
programs for several problems of statistical inference which are
difficult to approach otherwise, perform inference on the
programs, and analyze the results.  

\section{Intuition}
\label{sec:intuition}

To get an intuition behind stochastic conditioning, we take
a fresh look at the Beta-Bernoulli generative model:
\begin{equation}
	\begin{aligned}
		x & \sim \mathrm{Beta}(\alpha, \beta), 
                &
		y & \sim \mathrm{Bernoulli}(x).
	\end{aligned}
	\label{eqn:beta-bernoulli}
\end{equation}
The $\mathrm{Beta}$ prior on $x$ has $\alpha$
and $\beta$ parameters, which are interpreted as the belief about the
number of times $y{=}1$ and $y{=}0$ seen before. Since $\mathrm{Beta}$ is the conjugate prior
for $\mathrm{Bernoulli}$, belief updating
in~\eqref{eqn:beta-bernoulli} can be performed analytically:
\begin{equation*}
		x\vert y \sim \mathrm{Beta}(\alpha + y, \beta + 1 - y)
	\label{eqn:beta-bernoulli-x-given-y}
\end{equation*}
We can compose Bayesian belief updating. If after observing
$y$ we observed $y'$, then\footnote{By $y' \circ y$ we denote that
$y'$ was observed after observing $y$ and updating the belief
about the distribution of $x$.}
\begin{equation*}
	x\vert y' \circ y \sim \mathrm{Beta}(\alpha + y + y', \beta + 2 - y - y').
	\label{eqn:beta-bernoulli-x-given-yy}
\end{equation*}
In general, if we observe $y_{1:n} = y_n \circ ... \circ y_2 \circ y_1$, then
\begin{equation*}
        x\vert y_{1:n} \sim \mathrm{Beta}\Big(\alpha + \sum\nolimits_{i=1}^n y_i, \beta + n - \sum\nolimits_{i=1}^n y_i\Big).
	\label{eqn:beta-bernoulli-x-given-ys}
\end{equation*}
In \eqref{eqn:beta-bernoulli} (also in many more general
exchangeable settings) belief updating is commutative ---
the posterior distribution does not depend on the order of
observations. One may view $y_{1:n}$ as a multiset, rather than a
sequence, of observations. 

Let us now modify the procedure of presenting the evidence.
Instead of observing the value of each of $y_n \circ .... \circ y_2
\circ y_1$ in order, we just observe $n$ variates, of which
$k=\sum_{i=1}^n y_i$ variates have value $1$ (but we are
not told which ones). It does not matter which of the
observations are $1$ and which are $0$. We can even stretch the
notion of a single observation and say that it is, informally, a
`combination' of 1 with probability $\theta = \frac k n$ and 0
with probability $1 - \theta$. In other words, we can view each
observation $y_i$ as an observation of distribution
$\mathrm{Bernoulli}(\theta)$ itself; the posterior distribution of $x$
given $n$ observations of $\mathrm{Bernoulli}(\theta)$ should be
the same as the posterior distribution of $x$ given $y_{1:n}$.  This extended interpretation of belief updating
based on observing distributions lets us answer questions about
the posterior of $x$ given that we observe \textit{the
distribution} $\mathrm{Bernoulli}(\theta)$ of $y$:
\begin{equation*}
	x|\big(y \stoch \mathrm{Bernoulli}(\theta)\big) \sim \mathrm{Beta}(\alpha + \theta, \beta + 1 - \theta).
	\label{eqn:beta-bernoulli-x-given-bernoulli}
\end{equation*}
Note that observing a distribution does not imply observing its
parametric representation.  One may also observe a
distribution through a random source of samples, a black-box
unnormalized density function, or summary statistics.

Commonly, probabilistic programming involves weighing 
different assignments to $x$ by the conditional probability
of $y$ given $x$. For model~\eqref{eqn:beta-bernoulli},
\begin{equation*}
		p(y|x)  = x^y(1-x)^{1-y}.
\end{equation*}
The conditional probability of observing a fixed value extends
naturally to observing a distribution:
\begin{align}
		\label{eqn:beta-bernoulli-pD-given-x}
		\nonumber
		p&(y \stoch \mathrm{Bernoulli}(\theta)|x) = x^\theta (1-x)^{1-\theta} \\
		 & = \exp \left(\theta \log x + (1 - \theta) \log (1-x) \right) \\ \nonumber
         & = \exp \Big(\sum\nolimits_{y \in \{0,1\}} p_{\mathrm{Bern}(\theta)}(y) \cdot \log p_{\mathrm{Bern}(x)}(y)\Big)
\end{align}
where $p_{\mathrm{Bern}(r)}(y)$ is the probability mass function
of the distribution $\mathrm{Bernoulli}(r)$ evaluated at $y$. 
Note that $p_{\mathrm{Bern}(x)}$ inside the log is precisely the
conditional probability of $y$ in model \eqref{eqn:beta-bernoulli}.
Equation~\eqref{eqn:beta-bernoulli-pD-given-x} lets us specify a
probabilistic program for a version of
model~\eqref{eqn:beta-bernoulli} with stochastic conditioning
--- on a distribution rather than on a value. In the next
section, we introduce stochastic conditioning formally, using a
general form of~\eqref{eqn:beta-bernoulli-pD-given-x}.

\section{Stochastic Conditioning}
\label{sec:stochastic-conditioning}


Let us define stochastic conditioning formally. In what follows, we mostly discuss the
continuous case where the observed distribution $D$ has a density $q$. For the case that $D$
does not have a density, the notation $q(y)dy$ in our discussion should be replaced
with $D(dy)$, which means the Lebesgue integral with respect to the distribution (or probability measure)
$D$. For the discrete case, probability densities should be
replaced with probability masses, and integrals with sums. Modulo these changes, all the theorem and
propositions in this section carry over to the discrete case. A general measure-theoretic 
formalization of stochastic conditioning, which covers all of these cases uniformly, is described in 
Appendix~\ref{sec:stochastic-conditioning-generalization}.

\begin{dfn} 
	A probabilistic model with stochastic conditioning is a tuple $(p(x, y), D)$ where 
	(i) $p(x, y)$ is the joint probability density of random variable $x$ and observation $y$, and it is factored into
	the product of the prior $p(x)$ and the likelihood $p(y|x)$ (i.e., $p(x, y)=p(x)p(y|x)$);
	(ii) $D$  is the distribution from which observation $y$ is marginally sampled, and it has a density $q(y)$.
	\label{dfn:stochastic-conditioning}
\end{dfn}

Unlike in the usual setting, our objective is to infer $p(x|y \stoch D)$, the
distribution of $x$ given \textit{distribution} $D$, rather than an
individual observation $y$.  To accomplish this objective, we need to be
able to compute $p(x, y \stoch D)$, a possibly unnormalized density on $x$ and distribution $D$. 
We define $p(x, y \stoch D) = p(x)p(y \stoch D|x)$ where $p(y \stoch D|x)$ is the following unnormalized conditional density:
\begin{dfn}
	The (unnormalized) conditional density $p(y \stoch D|x)$ of $D$ given $x$ is
	\begin{equation}
		p(y \stoch D|x) = \exp \left( \int_Y (\log p(y|x))\,q(y)dy \right)
		\label{eqn:prob-D-given-x0}
	\end{equation}
	\label{dfn:prob-D-given-x}
	where $q$ is the density of $D$.
\end{dfn}

An intuition behind the definition can be seen by rewriting~\eqref{eqn:prob-D-given-x0} as a type II geometric integral:
\begin{equation}
	\label{eqn:geometric}
	\nonumber p(y \stoch D|x) = \prod\nolimits_Y p(y|x)^{q(y)dy}.
\end{equation}
Definition~\ref{dfn:prob-D-given-x} hence can be interpreted as the
probability of observing \textit{all} possible draws of $y$ from $D$,
each occurring according to its probability $q(y)dy$.

At this point, the reader may wonder why we do not take the following alternative, 
frequently coming up in discussions:
\begin{equation}
	p_1(y \stoch D|x) = \int_Y \,p(y|x)\, q(y)dy.
	\label{eqn:alternative-stochastic-conditioning}
\end{equation}
One may even see a connection
between~\eqref{eqn:alternative-stochastic-conditioning} and
Jeffrey's soft evidence~\cite{J90}
\begin{equation}
	p(x|y \soft D) = \int_Y q(y) p(x|y)\, dy, 
	\label{eqn:jeffreys-soft-evidence}
\end{equation}
although the latter addresses a different setting. In soft
evidence, an observation is a single value $y$, but the observer
does not know with certainty which of the values was observed.
Any value $y$ from $Y$, the domain of $D$, can be observed with
probability $q(y)$, but $p(y \soft D|x)$ cannot be
generally defined~\cite{CD03,BDP+13}. In our setting,
distribution $D$ is observed, and the observation is certain.

We have two reasons to
prefer~\eqref{eqn:prob-D-given-x0} to~\eqref{eqn:alternative-stochastic-conditioning}. First, as Proposition~\ref{pro:max-likelihood} will explain, our $p(y \stoch D|x)$ 
is closely related to the KL divergence between $q(y)$ and
$p(y|x)$, while the alternative $p_1(y \stoch D|x)$ in \eqref{eqn:alternative-stochastic-conditioning} lacks such
connection. The connection helps understand how $p(y \stoch D|x)$ alters the prior of $x$. Second, $p(y \stoch D|x)$ treats all possible
draws of $y$ more equally than $p_1(y \stoch D|x)$ in the
following sense. Both $p(y \stoch D|x)$ and $p_1(y \stoch D|x)$ are instances
of so called power mean~\cite{B03} defined by 
\begin{equation}
	p_\alpha(y \stoch D|x) = \left(\int_Y \,p(y|x)^\alpha\, q(y)dy\right)^{\frac{1}{\alpha}}.
\end{equation}
Setting $\alpha$ to $0$ and $1$ gives our $p(y \stoch D|x)$ and the alternative
$p_1(y \stoch D|x)$, respectively. A general property of this power mean
is that as $\alpha$ tends to $\infty$, draws of $y$ with large
$p(y|x)$ contribute more to $p_\alpha(y \stoch D|x)$, and 
the opposite situation happens as $\alpha$ tends to $-\infty$. Thus, the $\alpha = 0$ case, which gives our $p(y \stoch D|x)$, can 
be regarded as the option that is the least sensitive to $p(y|x)$.

\begin{pro} 
        For a given $D$ with density $q$, 
	\begin{equation}
		\arg \max_x p(y \stoch D|x) = \arg \min_x \mathrm{KL}(q(y)||p(y|x)).
		\label{eqn:prop-1}
	\end{equation}
	 In particular, if there exists $x^*$ such that
	$p(y|x^*) \equiv q(y)$, then $x^* = \arg \max_x p(y \stoch D|x)$.
	\label{pro:max-likelihood}
\end{pro}
\begin{proof}
	We prove the proposition by re-expressing $\log p(y \stoch D|x)$ in
	terms of the negative KL divergence:
	\begin{align*}
			& \log p(y \stoch D|x) = \int_Y (\log p(y|x))\,q(y)dy \\
			& \ {} = \left(\int_Y (\log q(y)) \,q(y)dy\right) - \int_Y \left(\log \frac{q(y)}{p(y|x)}\right) q(y)dy \\
			& \ {} = \left(\int_Y (\log q(y))\, q(y)dy\right) - \mathrm{KL}(q(y)||p(y|x)).
	\end{align*}
	Since the first term in the last line does not depend on $x$, we have the equation \eqref{eqn:prop-1}.
\end{proof}
The proposition does not hold for the alternative definition $p_1(y \stoch D|x)$ in \eqref{eqn:alternative-stochastic-conditioning}. 
Even if there exists $x^*$ such that $p(y|x^*) \equiv q(y)$, all we can say about $p(y|\arg\max_x
p_1(y \stoch D|x))$ is just that it maximizes $\int_Y p(y|x) \,q(y)dy$, not that it is
$q(y)$.  For example, consider the finite discrete case and assume that there also exists 
$x^\dag$ with $p(y|x^\dag) = \mathrm{Dirac}(\arg \max_y q(y))$. Then 
$x^\dag$, not $x^*$, maximizes $p_1(y \stoch D|x)$.

The next theorem explains our setting formally and shows that $p(y \stoch D|x)$ 
has a finite normalization constant.
\begin{thm}
	Assume that the distribution $D_\theta$ is parameterized by $\theta \in \Theta \subseteq \mathbb{R}^p$. Let $q_\theta$ be its density. Then, $p(y \stoch D_\theta|x)$ has a finite normalization constant $C$ over
	$\mathcal{D} = \{D_\theta \mid \theta \in \Theta\}$ if the following uniform-bound condition holds: $\sup_{y \in Y} \int_\Theta q_\theta(y)\,d\theta < \infty$.
	Thus, in this case, $p(y \stoch D_\theta|x) / C$ is a conditional probability 
	density on $\mathcal{D}$.
	\label{thm:well-defined-likelihood}	
\end{thm}
\begin{proof} 
	Let $C = \int_\Theta p(y \stoch D_\theta|x)\,d\theta$. Then,
	\begin{equation}
		C = \int_\Theta \exp \left( \int_Y (\log p(y|x))\, q_\theta(y)dy \right) d\theta.
	\end{equation}
	We compute a finite upper bound of $C$ as follows:
	\begin{align*}
			C 
			& \le^1 \int_\Theta \int_Y \exp(\log p(y|x))\, q_\theta(y)dy d\theta \\
			& {} = \int_\Theta \int_Y p(y|x) \,q_\theta(y)dy d\theta \\
			& {} =^2 \int_Y \left( \int_\Theta q_\theta(y)\, d\theta\right) p(y|x) dy \\
			& {} \leq \left(\sup\limits_{y' \in Y} \int_\Theta q_\theta(y')\,d\theta\right) \cdot \left(\int_Y  p(y|x)dy\right) \\
			& {} =  \sup\limits_{y' \in Y} \int_\Theta q_\theta(y')\,d\theta <^3 \infty.
	\end{align*}
	Here $\le^1$ is by Jensen's inequality, $=^2$ follows from Fubini's theorem, and $<^3$ uses the uniform-bound condition. 
\end{proof}

Let us illustrate the restriction  on the set of distributions $\mathcal{D}$
imposed by 
the uniform-bound condition in Theorem~\ref{thm:well-defined-likelihood}. 
The set $\mathcal{D} = \{\mathrm{Normal}(\theta, 1)  \mid \theta \in \mathbb{R}\}$ of normal distributions with the fixed variance $1$ meets the condition, so that when normalized, $p(y \stoch D|x)$ becomes a probability
density over $\mathcal{D}$. However, if we permit the variance to vary, the resulting set
$\mathcal{D}' = \{ \mathrm{Normal}(\theta,\sigma^2) \mid \theta \in \mathbb{R}, \sigma^2 \in (0,\infty)\}$ does not satisfy
the condition. 
Appendix~\ref{sec:stochastic-conditioning-generalization} contains another
set that violates a measure-theoretic generalization of our condition 
(described in Appendix~\ref{sec:stochastic-conditioning-generalization} as well) 
and, furthermore, does not have a normalization constant for
some choice of $p$.

The next proposition shows that stochastic conditioning
generalizes conventional conditioning on a value. It uses a
non-density version of
Definition~\ref{dfn:stochastic-conditioning} where the integral
over $q(y)dy$ is understood as the integral over the
distribution (i.e., probability measure) $D$.
\begin{pro}
	When observing a distribution reduces to
	observing a single value, $D = \mathrm{Dirac}(y)$,
	conditioning on a distribution reduces to conventional
	conditioning on a value: $p(\mathrm{Dirac}(y)|x) = p(y|x)$.
	\label{pro:dirac}
\end{pro}

\section{Inference}
\label{sec:inference}

Algorithms for deterministic conditioning cannot be applied
without modification to probabilistic programs with  stochastic
conditioning.  One approach is to rely on nested Monte Carlo
estimation~\citep{RCY+18}. However, probabilistic programs with
stochastic conditioning constitute an important special case of
nested models, and it is possible to leverage properties of such
programs to apply a wider class of inference algorithms. In our
setting, $\log p(y \stoch D|x)$ is easy to estimate, and $p(y \stoch D|x)$ can be
estimated in a bias-adjusted manner, hence inference algorithms
using these estimates can be applied effectively.


In particular, Definition~\ref{dfn:prob-D-given-x} implies that an unbiased
Monte Carlo estimate of log likelihood $\log p(y \stoch D|x)$ is available based on samples
$y_i \sim D$ through which $D$ is observed:
\begin{equation}
	\log p(y \stoch D|x) \approx \frac 1 N \sum\nolimits_{i=1}^N \log p(y_i|x)
	\label{eqn:log-p-mc-stochastic}
\end{equation}
Another setting in which an unbiased Monte Carlo estimate of log
likelihood is available is subsampling for inference in models
with tall data~\cite{KCW14,BDH14,BDH17,MA14,QVK+18,QKV+19,DQK+19}. In
models considered for subsampling, $K$ observations $y_1, y_2,
..., y_K$ are conditionally independent given $x$:
\begin{equation}
	p(y_1, y_2, ..., y_K|x) = \prod\nolimits_{i=1}^K p(y_i|x)
\end{equation}
Most inference algorithms require evaluation of likelihood $p(y_1, y_2,
..., y_K|x)$, which is expensive if $K$ is large. For example,
in importance sampling, the likelihood is involved in the
computation of importance weights. In many Markov chain Monte Carlo
methods, the ratio of likelihoods of the proposed and the
current state is a factor in the Metropolis-Hastings acceptance
rate. Subsampling replaces $p(y_1, y_2, ..., y_K|x)$ by
an estimate based on $N$ samples $y_{i_1}, y_{i_2}, ...,
y_{i_N}$, $N < K$, which results in an unbiased Monte Carlo
estimate of log likelihood:
\begin{equation}
	\log p(y_1, y_2, ..., y_K|x) \approx \frac K N \sum\nolimits_{j=1}^N \log p(y_{i_j}|x)
	\label{eqn:log-p-mc-subsampling}
\end{equation}
The only difference between \eqref{eqn:log-p-mc-stochastic} and
\eqref{eqn:log-p-mc-subsampling} is in factor $K$, and inference
algorithms for subsampling can be applied to stochastic
conditioning with minor modifications.  

A simple bias-adjusted likelihood estimate $\hat p(x, y \stoch
D)$, required for the computation of the weights in importance
sampling as well as of the acceptance ratio in pseudo-marginal
Markov chain Monte Carlo~\cite{AR09}, can be computed based on
\eqref{eqn:log-p-mc-stochastic}~\cite{CD99,NFW12,QVK+18}.
Stochastic gradient-based inference
algorithms~\cite{CFG14,MCF15,HBW+13,RGB14,KTR+17} rely on an
unbiased estimate of the gradient of log likelihood, which is
trivially obtained by differentiating both sides
of~\eqref{eqn:log-p-mc-stochastic}.

We implemented inference in probabilistic programs with
stochastic conditioning for Infergo~\cite{T19}.
To facilitate support for stochastic conditioning in
other probabilistic programming systems, we provide details on
likelihood estimation and some possible adaptations of inference
algorithms to stochastic conditioning, as well as pointers to
alternative adaptations in the context of subsampling, in
Appendix~\ref{app:algorithms}.

\section{Case Studies}
\label{sec:case-studies}

In the case studies, we explore several problems cast as
probabilistic programs with stochastic conditioning. We place $y
\stochsim D$ above a rule to denote that distribution $D$ is
observed through $y$ and is otherwise unknown to the model, as
in~\eqref{eqn:notation-sampling}.  Some models are more natural
to express in terms of the joint probability that they compute
than in terms of distributions from which $x$ is drawn and $y$
is observed. In that case, we put the expression for the joint
probability $p(x, y)$ under the rule, as
in~\eqref{eqn:notation-scoring}.\\
\begin{minipage}{0.55\linewidth}
\vspace{-0.5\baselineskip}
\begin{equation}
	\begin{aligned}
		y & \stochsim D \\ \midrule
		x & \sim \textit{Prior} \\
		y\vert x & \sim \textit{Conditional}(x)
	\end{aligned}
	\label{eqn:notation-sampling}
\end{equation}
\end{minipage}
\begin{minipage}{0.44\linewidth}
\vspace{-0.5\baselineskip}
\begin{equation}
	\begin{aligned}
		y & \stochsim D \\ \midrule
		p & (x, y) = ...
	\end{aligned}
	\label{eqn:notation-scoring}
\end{equation}
\end{minipage}

The code and data for the case studies are provided in
repository~\url{https://bitbucket.org/dtolpin/stochastic-conditioning}.

\subsection{Inferring the Accuracy of Weather Forecast}
\label{sec:commute}

A person commutes to work either by
motorcycle or, on rainy days, by taxi. When the weather is good,
the motorcycle ride takes $15\pm 2$ minutes via a highway. If rain is
expected, the commuter takes
a taxi, and the trip takes $30\pm 4$ minutes, because of crowded roads
which slow down a four-wheeled vehicle. Sometimes, however, the
rain catches the commuter in the saddle, and 
the commuter rides slowly and carefully through rural roads,
arriving at $60\pm 8$ minutes. Given weather observations and
trip durations, we want to estimate the accuracy of rain
forecasts, that is, the probability of the positive forecast on
rainy days $p_t$ (true positive) and on dry days $p_f$ (false
positive).

The problem is represented by the following model:
\begin{align}
	\vspace{-0.5\baselineskip}
    \label{eqn:commute-genmod}
		\nonumber
		p_r, p_t, p_f& \sim \textrm{Beta}(1, 1)\\
		\var{rain}\vert p_r & \sim \textrm{Bernoulli}(p_r) \\ \nonumber
        \var{willRain}\vert p_t, p_f, \var{rain} & \sim \begin{cases}
					 & \!\!\!\!\!\!\textrm{Bernoulli}(p_t)\ \textrm{if}\ \var{rain} \\
            & \!\!\!\!\!\!\textrm{Bernoulli}(p_f)\ \textrm{otherwise}
        \end{cases} \\ \nonumber
        \var{duration}\vert \var{rain}, \var{willRain} & \sim \begin{cases}
                         & \!\!\!\!\!\!\textrm{Normal(30, 4)}\ \textrm{if}\ \var{willRain} \\
                         & \!\!\!\!\!\!\textrm{Normal(15, 2)}\ \textrm{if}\ \lnot \var{rain} \\
					     & \!\!\!\!\!\!\textrm{Normal(60, 8)}\ \textrm{otherwise}
        \end{cases}
\end{align}
Model~\eqref{eqn:commute-genmod} can be interpreted as either a simulator
that draws samples of (rain, duration) given $p_r$, $p_t$, and
$p_f$, or as a procedure that computes the conditional
probability of (rain, duration) given $p_r$, $p_t$, and $p_f$.
We use the simulator interpretation to generate synthetic
observations for 30 days and $p_r=0.2$, $p_t=0.8$, $p_f=0.1$.
The conditional probability interpretation lets us write down a
probabilistic program for posterior inference of $p_t$ and $p_f$
given observations. 

If, instead of observing (rain, duration) simultaneously,
we observe weather conditions and trip durations
separately and do not know correspondence between them
(a common situation when measurements are collected
by different parties), we can still
write a conventional probabilistic program conditioned on the
Cartesian product of weather conditions
and trip durations, but the number of observations and,
thus, 
time complexity of inference becomes quadratic in
the number of days. 
In general, when a model is conditioned on the Cartesian product
of separately obtained observation sets, inference complexity
grows exponentially with the dimensionality of observations, and
inference is infeasible in problems with more than a couple of
observed features.

\begin{figure}
	\begin{subfigure}{0.495\linewidth}
		\includegraphics[width=0.95\linewidth]{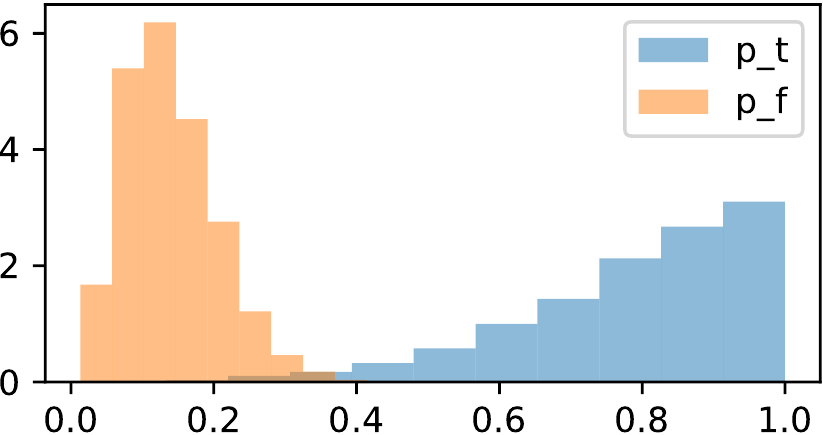}
		\caption{deterministic}
		\label{fig:commute-posteriors-deterministic}
	\end{subfigure}
	\begin{subfigure}{0.495\linewidth}
		\includegraphics[width=0.95\linewidth]{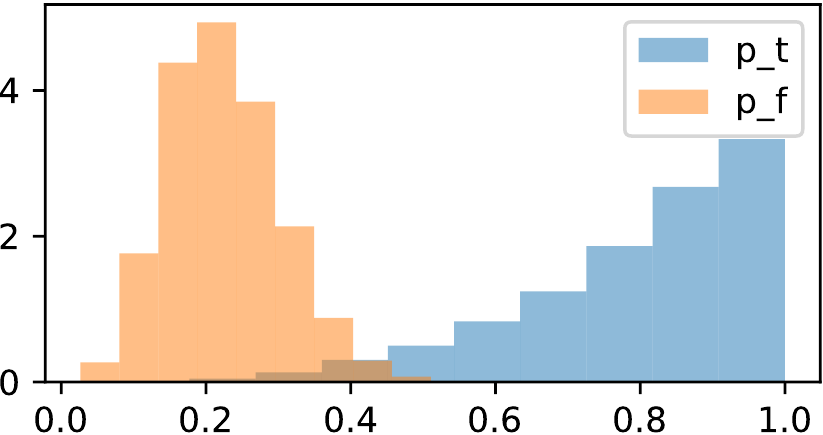}
		\caption{averaged}
		\label{fig:commute-posteriors-averaged}
	\end{subfigure}
	\begin{subfigure}{0.495\linewidth}
		\includegraphics[width=0.95\linewidth]{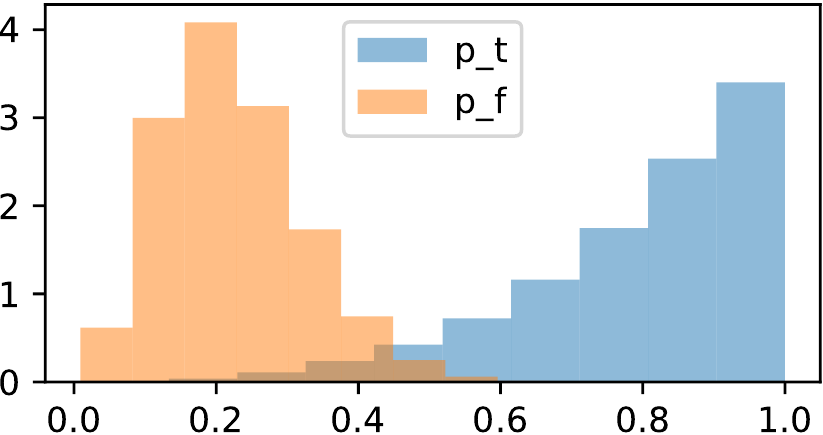}
		\caption{stochastic}
		\label{fig:commute-posteriors-stochastic}
	\end{subfigure}
	\begin{subfigure}{0.495\linewidth}
		\includegraphics[width=0.95\linewidth]{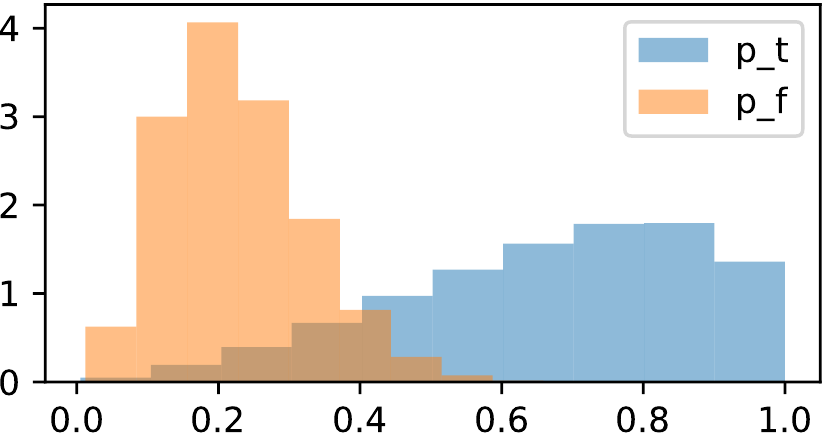}
		\caption{intensity}
		\label{fig:commute-posteriors-intensity}
	\end{subfigure}
	\vspace{-0.5\baselineskip}
    \caption{Commute to work: posteriors of $p_t$ and $p_f$ for
	each of the four models. Despite being exposed to partial
	information only, models with stochastic conditioning let us
	infer informative posteriors.}
	\vspace{-\baselineskip}
    \label{fig:commute-posteriors}
\end{figure}

Alternatively, we can draw rain and duration from the
observation sets randomly and independently, and stochastically
condition on $D = \var{Rains} \times \var{Durations}$:
\begin{equation}
	\vspace{-0.5\baselineskip}
    \begin{aligned}
            & \var{rain}, \var{duration} \stochsim \var{Rains} \times \var{Durations} \\ \midrule
		& ...
    \end{aligned}
    \label{eqn:commute-genmod-sc}
\end{equation}
One can argue that the probabilistic program for the case of
independent sets of observations of rain and duration can still
be implemented with linear complexity by noting that the domain
of rain contains only two values, true and false, and
analytically averaging the evidence over rain.  However, such
averaging is often impossible. Consider a variant of the
problem in which the duration of a motorcycle trip in rain
depends on rain intensity.  Stochastic conditioning, along with
inference algorithms that use a small number of samples to
estimate the log likelihood (magnitude or gradient), lets us
preserve linear complexity in the number of
observations~\cite{DPD15,BDH17}.

We fit the model using stochastic gradient Hamiltonian Monte
Carlo and used $10\,000$ samples to approximate the posterior.
Figure~\ref{fig:commute-posteriors} shows marginal posteriors of
$p_t$ and $p_f$ for each of the four models, on the same
simulated data set.  Posterior distributions should be the same
for the analytically averaged and stochastic models. The
deterministic model is exposed to more information
(correspondence between rain occurrence and trip duration).
Hence, the posterior distributions are more sharply peaked. The
stochastic model with observation of intensity should be less
confident about $p_t$, since now the observation of a motorcycle
trip duration slowed down by rain is supposed to come from a
distribution conditioned on rain intensity.

\subsection{Estimating the Population of New York State}
\label{sec:case-studies-population}
This case study is inspired by~\citet{R83}, also appearing as
Section~7.6 in~\citet{GCS+13}. The original 
case 
study evaluated
Bayesian inference on the problem of estimating the total
population of 804 municipalities of New York state based on a
sample of 100 municipalities. Two samples were given, with
different summary statistics, and power-transformed normal model
was fit to the data to make predictions consistent among the
samples. The authors of the original 
study apparently had
access to the full data set (population of each of 804
municipalities).  However, only summary description of the
samples appears in the publication: mean, standard
deviation, and quantiles (Table~\ref{tab:nypopu-data}). We
show how such summary description can be used to perform
Bayesian inference, with the help of stochastic
conditioning.

\begin{table}
	\centering
	\caption{Summary statistics for populations of
	municipalities in New York State in 1960; all 804 municipalities
	and two random samples of 100.
	From~\citet{R83}.}
	\label{tab:nypopu-data}
	\vspace{-0.5\baselineskip}
	\begin{tabular}{l r r r}
		& Population & Sample 1  & Sample 2  \\ \hline
total    &13,776,663   &1,966,745  &3,850,502 \\
mean     &17,135      &19,667    &38,505 \\
sd       &139,147     &142,218   &228,625 \\
lowest   &19         &164      &162 \\
5\%      &336        &308      &315 \\
25\%     &800        &891      &863 \\
median   &1,668       &2,081     &1,740 \\
75\%     &5,050       &6,049     &5,239 \\
95\%     &30,295      &25,130    &41,718 \\
highest  &2,627,319    &1,424,815  &1809578
	\end{tabular}
	\vspace{-\baselineskip}
\end{table}

The original case study in~\citet{R83} started with comparing
normal and log-normal models, and finally fit a truncated three-parameter
power-transformed normal distribution to the data, which helped
reconcile conclusions based on each of the samples while
producing results consistent with the total population. Here,
we use a model with log-normal sampling distribution, the normal
prior on the mean, based on the summary statistics,
and the improper uniform prior on the log of the variance.
 To complete the model, we stochastically condition on the
piecewise-uniform distribution $D$ of municipality populations 
according to the quantiles:
\begin{align}
	\vspace{-0.5\baselineskip}
	\label{eqn:nypopu} \nonumber
        & \ y_{1\ldots n} \stochsim \var{Quantiles} \\ \midrule \nonumber
        & \ m\! \sim\! \mathrm{Normal}\Big(\var{mean}, {\var{sd}}/ {\sqrt n}\Big),\,\log s^2\!\sim\!\mathrm{Uniform}(-\infty,\!\infty) \\
        & \ \sigma = \sqrt{\log \left(s^2/m^2 + 1\right)},\quad\mu  = \log m - {\sigma^2} / 2 \\ \nonumber
        & \ y_{1\ldots n}\vert m,s^2 \sim  \mathrm{LogNormal}(\mu, \sigma)
\end{align}

\begin{figure}
	\centering
	\includegraphics[width=0.95\linewidth]{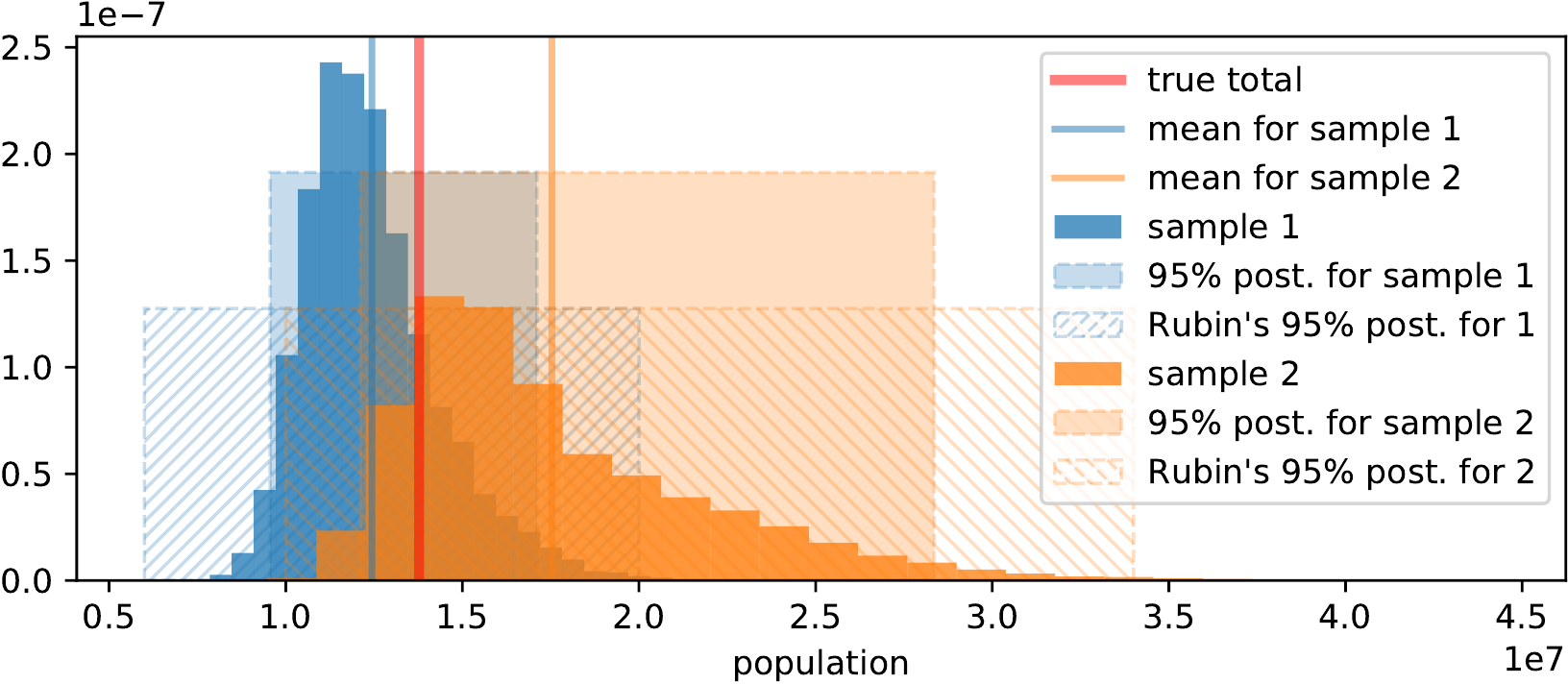}
	\vspace{-0.5\baselineskip}
	\caption{Estimating the population of NY state. 95\%
	intervals inferred from the summary statistics include the
	true total, and are tighter than Rubin's results.}
	\vspace{-\baselineskip}
	\label{fig:nypopu}
\end{figure}

As in Section~\ref{sec:commute}, we fit the model using
stochastic gradient HMC and used $10\,000$
samples to approximate the posterior.  We then used 
$10\,000$ draws with replacement of $804$-element sample sets
from the predictive posterior to estimate the total population.
The posterior predictive distributions of the total population
from both samples are shown in Figure~\ref{fig:nypopu}. The 95\%
intervals inferred from the summary statistics, $[9.6\times
10^6, 17.2 \times 10^6]$ for sample 1, $[12.1\times 10^6, 28.1
\times 10^6]$ for sample 2, cover the true total $13.8\times
10^6$, and are tighter than the best intervals based on the full
samples reported by \citet{R83}, $[6 \times 10^6, 20 \times
10^6]$ for sample 1 and $[10 \times 10^6, 34\times 10^6]$ for
sample 2.

\subsection{The Sailing Problem}

\begin{figure}
	\begin{subfigure}{0.495\linewidth}
		\centering
		\includegraphics[width=0.8\linewidth]{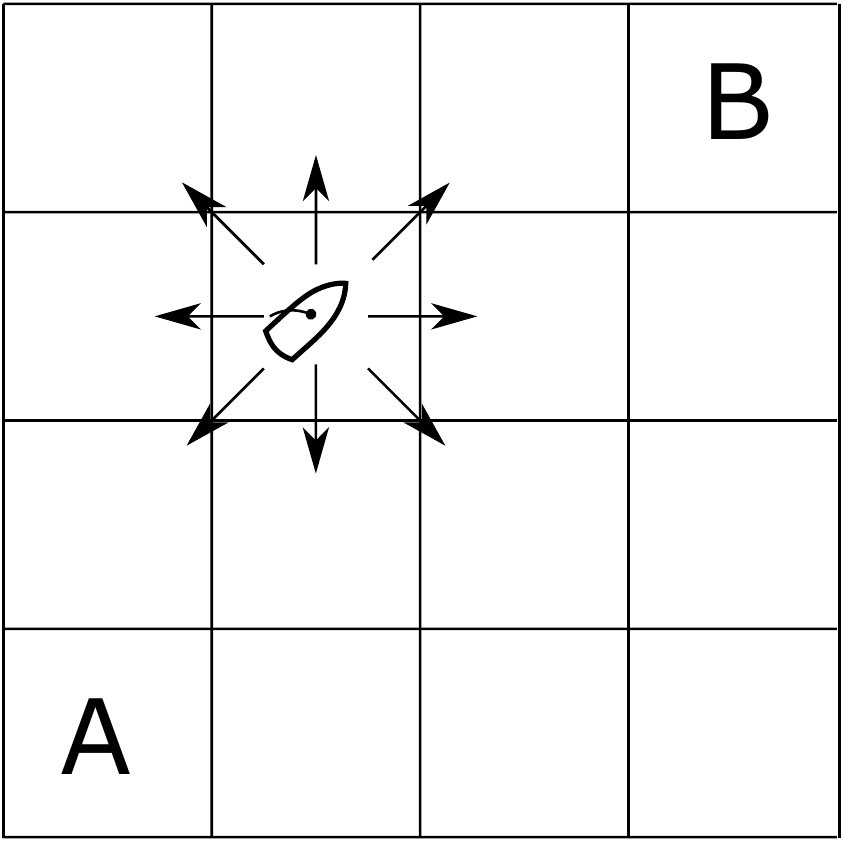}
		\caption{lake}
		\label{fig:sailing-lake}
	\end{subfigure}
	\begin{subfigure}{0.495\linewidth}
		\centering
		\includegraphics[width=0.8\linewidth]{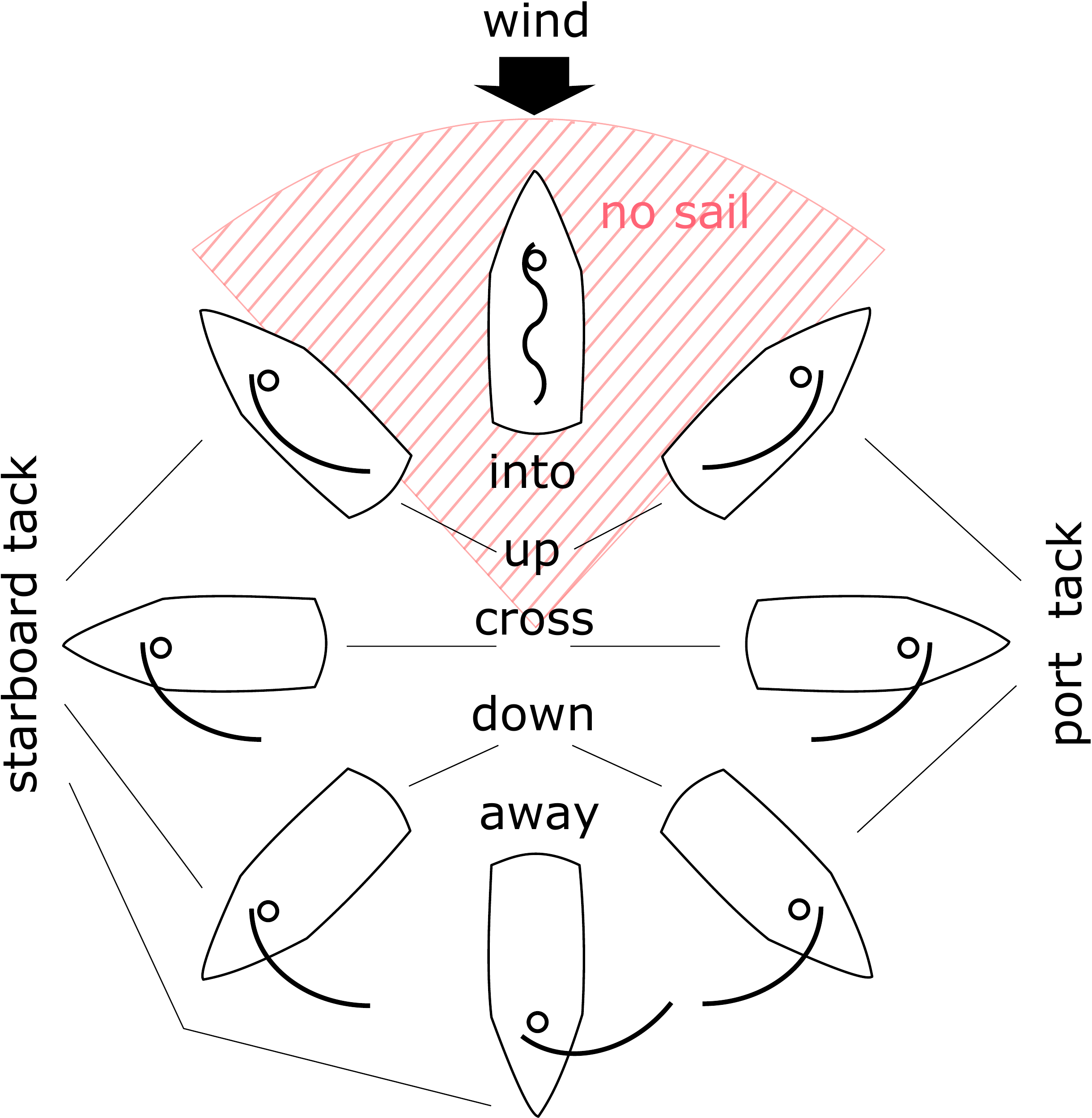}
		\caption{points of sail}
		\label{fig:sailing-points}
	\end{subfigure}
	\vspace{-0.5\baselineskip}
	\caption{The sailing problem}
	\vspace{-\baselineskip}
	\label{fig:sailing}
\end{figure}

The sailing problem (Figure~\ref{fig:sailing}) is a popular
benchmark problem for search and planning~\cite{PG04,KS06,TS12}.
A sailing boat must travel
between the opposite corners A and B of a square lake of a
given size. At each step, the boat can \textit{head} in 8
directions (\textit{legs}) to adjacent squares
(Figure~\ref{fig:sailing-lake}). The
unit distance cost of movement depends on the wind
(Figure~\ref{fig:sailing-points}), which can
also blow in 8 directions. There are five relative
boat and wind directions and associated costs: \textit{into},
\textit{up}, \textit{cross}, \textit{down}, and \textit{away}. 
The cost of sailing into the wind is prohibitively high, upwind
is the highest feasible, and away from the wind is the lowest.
The side of the boat off which the sail is hanging is called the
\textit{tack}. 
When the angle between the boat and the wind changes sign, the
sail must be \textit{tacked} to the opposite tack,
which incurs an additional \textit{tacking delay} cost.  The
objective is to find a policy that minimizes the expected travel
cost. The wind is assumed to follow a random walk, either
staying the same or switching to an adjacent
direction, with a known probability.

For any given lake size, the optimal policy can be found using
value iteration~\cite{B57}. The optimal policy is non-parametric: it
tabulates the leg for each combination of location, tack, and
wind. In this case study, we learn a simple parametric policy,
which chooses a leg that maximizes the sum of the leg
cost and the remaining travel cost after the leg, estimated as
the Euclidean distance to the goal multiplied by the average
unit distance cost:
\begin{equation}
	\vspace{-0.5\baselineskip}
	\begin{aligned}
		\var{leg}  = & \arg \min{}_{\var{leg}} \big[ \var{cost}(\var{tack}, \var{leg}, \var{wind}) + \\ 
		 & \var{unit-cost}\cdot\var{distance}(\var{next-location}, \var{goal})\big]
	\end{aligned}
	\label{eqn:sailing-policy}
\end{equation}

The average unit distance cost is the policy variable which we
infer. Model~\eqref{eqn:sailing} formalizes our setting.
Stochastic conditioning on $D=\var{RandomWalk}$ models non-determinism in
wind directions.
\begin{equation}
	\vspace{-0.5\baselineskip}
	\begin{aligned}
                & \qquad\qquad\quad\ \, \var{wind-history} \stochsim \var{RandomWalk} \\ \midrule
                & p(\var{wind-history}, \var{unit-cost}) = {}
                \\
                & \ \, \frac{1}{Z} \exp\Big(\frac {- \var{travel-cost}(\var{wind-history},\var{unit-cost})} {\var{lake-size} \cdot \var{temperature}}\Big)
	\end{aligned}
	\label{eqn:sailing}
\end{equation}
Under policy~\eqref{eqn:sailing-policy}, the boat trajectory and
the travel cost are determined by the wind history and the unit cost.
The joint probability of the wind history and the unit cost is given by
the Boltzmann distribution of trajectories with the travel cost
as the energy, a common physics-inspired choice in stochastic
control and policy search~\cite{K07,WGR+11,MPT+16}. The
temperature is a model parameter: the lower the temperature is,
the tighter is the concentration of policies around the optimal
policy.  A uniform prior on the unit cost, within a feasible
range, is implicitly assumed. If desirable, an informative
prior can be added as a factor depending on the unit cost.

\begin{table}
	\caption{Sailing problem parameters}
	\vspace{-0.5\baselineskip}
	\label{tab:sailing-parameters}
	\setlength\tabcolsep{3pt}
	\centering
	\begin{tabular}{c c c c c c | c c c}
		\multicolumn{6}{c|}{cost} & \multicolumn{3}{c}{wind probability} \\
		into & up & cross & down & away & delay & same & left & right \\ \hline
		$\infty$ & 4  & 3 & 2 & 1 & 4 & 0.4 & 0.3 & 0.3
	\end{tabular}
	\vspace{-\baselineskip}
\end{table}

The model parameters (cost and wind change
probabilities), same as  in~\citet{KS06,TS12}, are shown in
Table~\ref{tab:sailing-parameters}.
We fit the model using pseudo-marginal
Metropolis-Hastings~\cite{AR09} and used $10\,000$ samples to
approximate the posterior. The inferred unit and expected travel
costs are shown in Figure~\ref{fig:sailing-posterior}.
Figure~\ref{fig:sailing-posterior-unit} shows the posterior distribution
of the unit cost, for two 
temperatures. For all lake
sizes in the experiment (25, 50, 100), the optimal unit cost,
corresponding to the mode of the posterior, is $\approx
3.5$--$3.9$.  Distributions for lower temperatures are tighter
around the mode.  Figure~\ref{fig:sailing-posterior-travel} shows the
expected travel costs, with the expectations estimated both over
the unit cost and the wind. The 95\% posterior intervals are
shaded. The inferred travel costs are compared to the travel
costs of the optimal policy (the dashed line of the same color)
and of the greedy policy (the dotted line of the same color),
according to which the boat always heads in the direction of
the steepest decrease of the distance to the goal. One can see
that the inferred policies attain a lower expected travel cost
than the greedy policy and become closer to the optimal policy
as the temperature decreases.

\begin{figure}
	\begin{subfigure}{\linewidth}
		\centering
		\includegraphics[width=0.95\linewidth]{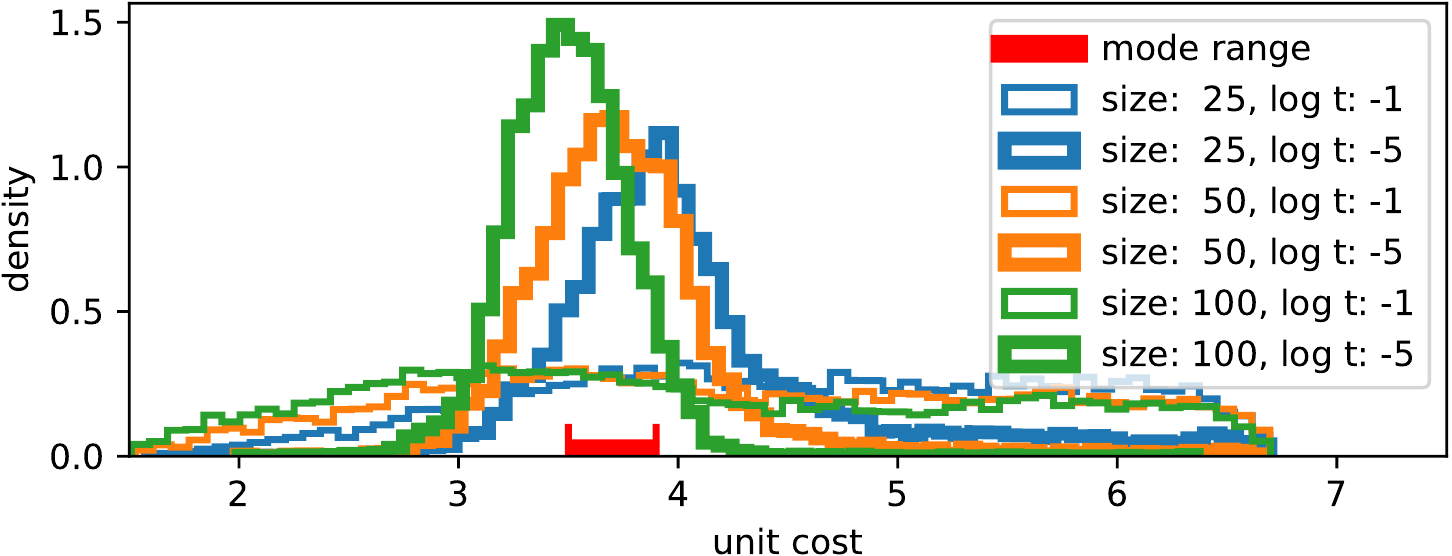}
		\caption{unit cost}
		\label{fig:sailing-posterior-unit}
	\end{subfigure}
	\begin{subfigure}{\linewidth}
		\centering
		\includegraphics[width=0.95\linewidth]{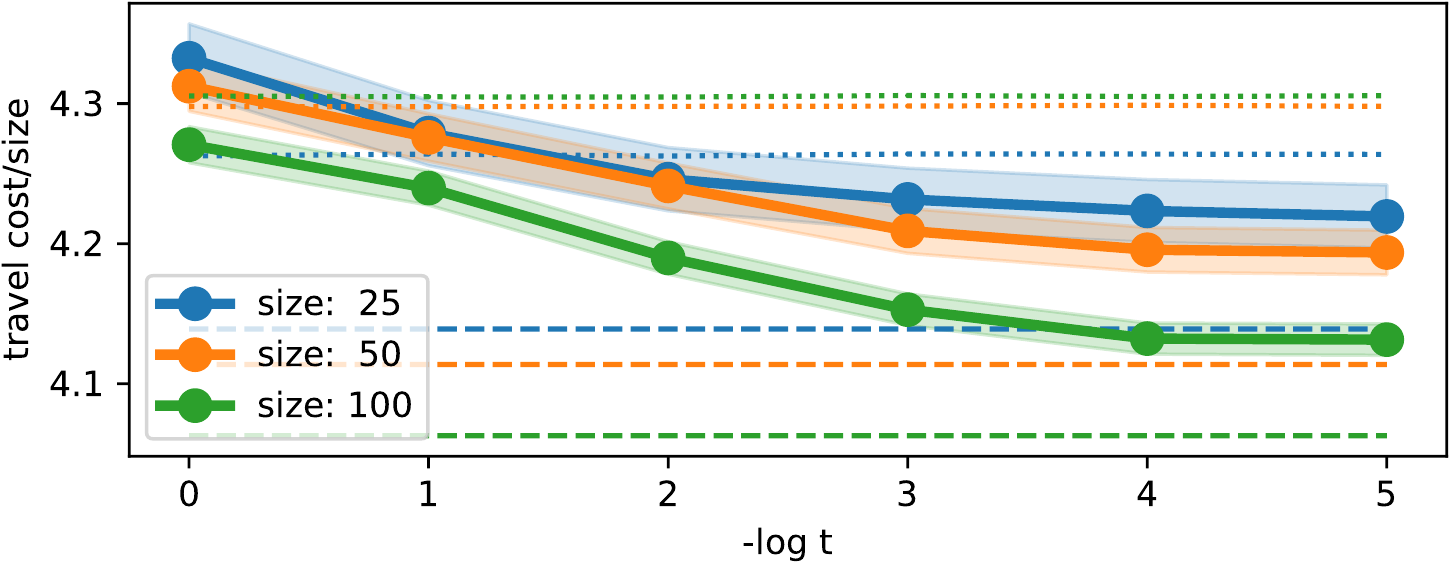}
		\caption{expected travel cost}
		\label{fig:sailing-posterior-travel}
	\end{subfigure}
	\vspace{-0.5\baselineskip}
	\caption{The sailing problem. The optimum unit cost is
	$\approx 3.5$--$3.9$. The dashed lines
	are the expected travel costs of the optimal policies, dotted
	--- of the greedy policy.}
	\vspace{-\baselineskip}
	\label{fig:sailing-posterior}
\end{figure}

\section{Related Work}

Works related to this research belong to several interconnected
areas: non-determinism in probabilistic programs,  nesting of
probabilistic programs, inference in nested statistical models,
and conditioning on distributions.

Stochastic conditioning can be viewed as an expression of
non-determinism with regard to the observed variate. The problem
of representing and handling non-determinism in
probabilistic programs was raised in~\citet{GHNR14}, as an
avenue for future work. Non-determinism arises, in particular,
in application of probabilistic programming to policy search in
stochastic domains. \citet{MPT+16} introduce a policy-search
specific model specification and inference scheme based on
black-box variational inference. We suggest, and show in a case
study, that policy search in stochastic domains can be cast as
inference in probabilistic programs with stochastic
conditioning. 

It was noted that probabilistic programs, or queries, can be
nested, and that nested probabilistic programs are able to
represent models beyond those representable by flat
probabilistic programs. \citet{SG14} describe how probabilistic
programs can represent nested conditioning as a part of the
model, with examples in diverse areas of game theory, artificial
intelligence, and linguistics. \citet{SMW18} apply nested
probabilistic programs to reasoning about autonomous agents.
Some probabilistic programming languages such as Church
\citep{GMR+08}, WebPPL \citep{GS15}, Anglican \citep{TMY+16},
and Gen \citep{CSL+19} support nesting of probabilistic
programs.  Stochastic conditioning can be, in principle,
represented through nesting, however nesting in general incurs
difficulties in inference~\citep{RCY+18,R18}\nocite{SYW+16}.
Stochastic conditioning, introduced in this work, allows both
simpler specification and more efficient inference, eliminating
the need for nesting in many important cases.

Conditioning of statistical models on distributions or
distributional properties is broadly used in machine
learning~\citep{CG01,KW19,GPM+14,MSJ+15,BCJ+19}. Conditioning on
distributions represented by samples is related to subsampling
in deep probabilistic programming~\citep{THS+17,TFP18,BCJ+19}.
Subsampling used with stochastic variational
inference~\citep{RGB14} can be interpreted as a special case of
stochastic conditioning. \citet{TZM+19} approach the problem of
conditioning on distributions by extending probabilistic
programming language \textsc{Omega} with constructs for
conditioning on distributional properties such as expectation or
variance. This work takes a different approach by generalizing
deterministic conditioning on values to stochastic conditioning
on distributions, without the need to explicitly compute or
estimate particular distributional properties, and leverages
inference algorithms developed in the context of
subsampling~\cite{KCW14,BDH14,BDH17,MA14,QVK+18,QKV+19,DQK+19}
for efficient inference in probabilistic programs with
stochastic conditioning.

There is a connection between stochastic conditioning and
Jeffrey's soft evidence~\citep{J90}. In soft evidence, the
observation is uncertain; any one out of a set of observations
could have been observed with a certain known probability. A
related concept in the context of belief networks is virtual
evidence~\citep{P88}. \citet{CD03} demonstrate that Jeffrey's
soft evidence and Pearl's virtual evidence are different
formulations of the same concept.  \citet{BDP+13,DLB16,J18}
elaborate on connection between soft and virtual evidence and
their role in probabilistic inference.  In probabilistic
programming, some cases of soft conditioning~\cite{WVM14,GS15}
can be interpreted as soft evidence.  In this work, the setting
is different: a distribution is observed, and the observation is
certain.

\section{Discussion}

In this work, we introduced the notion of stochastic
conditioning. We described kinds of problems for which
deterministic conditioning is insufficient, and showed on case
studies how probabilistic programs with stochastic conditioning
can be used to represent and efficiently analyze such problems.
We believe that adoption of stochastic conditioning in
probabilistic programming frameworks will facilitate convenient
modeling of new classes of problems, while still supporting
robust and efficient inference.  The idea of stochastic
conditioning is very general, and we believe our work opens up a
wide array of new research directions because of this.  Support
for stochastic conditioning in other existing probabilistic
programming languages and libraries is a direction for future
work.  While we provide a reference implementation, used in the
case studies, we believe that stochastic conditioning should
eventually become a part of most probabilistic programming
frameworks, just like other common core concepts.


\clearpage
\bibliography{refs}
\bibliographystyle{icml2021}

\clearpage
\appendix

\section{Measure-Theoretic Formalization of Stochastic Conditioning}
\label{sec:stochastic-conditioning-generalization}

Although stochastic conditioning is defined in terms of the density $q$ of the distribution $D$, its key idea does not depend on $q$. In fact, we have already explained informally how stochastic conditioning and our results can be developed even when the density $q$ does not exist, as in the case of Dirac distributions. Also, Proposition~\ref{pro:dirac} assumes this general development. In this section, we spell out this informal explanation, and describe the measure-theoretic formalization of stochastic conditioning.

We start by changing Definitions~\ref{dfn:stochastic-conditioning} and \ref{dfn:prob-D-given-x} such that $D$ is not required to have a density with respect to the Lebesgue measure, and the conditional density $p(y \stoch D|x)$ is defined for such $D$.
\begin{dfn} 
	A probabilistic model with stochastic conditioning is a tuple $(p(x, y), D)$ where 
	\begin{itemize}
		\item $p(x, y)$ is the joint probability density of random
			variable $x$ and observation $y$, and it is factored into
			the product of the prior $p(x)$ and the conditional probability $p(y|x)$ (i.e., $p(x, y)=p(x)p(y|x)$);
		\item $D$  is the distribution (i.e., probability measure) from which observation $y$ is sampled.
		\end{itemize}
	\label{dfn:stochastic-conditioning-generalization}
\end{dfn}
\begin{dfn}
	The conditional density $p(y \stoch D|x)$ of $D$ given $x$ is
	\begin{equation}
		p(y \stoch D|x) = \exp \left( \int_Y (\log p(y|x))\, D(dy) \right)
		\label{eqn:prob-D-given-x0-generalization}
	\end{equation}
	where $D(dy)$ indicates that the integral over $Y$ is taken with 
	respect to the distribution $D$.
	\label{dfn:prob-D-given-x-generalization}
\end{dfn}

To explain where the term ``density'' in Definition~\ref{dfn:prob-D-given-x-generalization} comes from, we recall the standard setup of 
the work on random distributions, which studies distributions over distributions.\footnote{A brief yet good exposition on this topic can be found in Appendix A of \citet{GhosalBayesianBook2017}.} The setup over random distributions on $Y \subseteq \mathbb{R}^m$ is the measurable space $(\mathcal{D},\Sigma)$ where $\mathcal{D}$ is 
the set of distributions over $Y$ and $\Sigma$ is the smallest $\sigma$-field 
generated by the family
\[
\Big\{\{D \mid D(A) < r\} \;\Big|\; \text{measurable $A \subseteq Y$ and $r \in \mathbb{R}$}\Big\}.
\]
The next theorem generalizes Theorem~\ref{thm:well-defined-likelihood}. In a setting that covers both continuous and discrete cases, with or without densities, the theorem describes when $p(y \stoch D|x)$ has a finite normalization constant.
\begin{thm}
	Assume that we are given a distribution $D_\theta$ parameterized by $\theta \in \Theta \subseteq \mathbb{R}^p$
	such that $D$ is a probability kernel from $\Theta$ to $Y$, and the following $\mu_x$ is a well-defined
	unnormalized distribution (i.e., measure) over $\Theta$: for all measurable subsets $B$ of $\Theta$,
	\begin{align*}
		\mu_x(B) 
		& {} = \int_B p(y \stoch D_\theta | x)\, d\theta 
		\\
		& {} = \int_B \exp \left(\int_Y (\log p(y|x)) \,D_\theta(dy)\right) d\theta.
	\end{align*}
	Let $\nu_x$ be the push-forward of $\mu_x$ along the function
	$\theta \longmapsto D_\theta$ from $\Theta$ to $\mathcal{D}$. The unnormalized distribution $\nu_x$
	has a finite normalization constant $C$ (i.e., $\nu_x(\mathcal{D}) = C < \infty$) if there exists $C' < \infty$
	such that for all measurable subsets
	$A$ of $Y$,
	\begin{equation}
		\label{eqn:measure-bound-requirement}
		\int_\Theta D_\theta(A)\, d\theta \leq \left(C' \cdot \int_A dy\right).
	\end{equation}
	\label{thm:measure}	
\end{thm}
Before proving the theorem, we make two comments. First,
when $D_\theta$ is defined in terms of a density $q_\theta$, the condition \eqref{eqn:measure-bound-requirement} in the theorem 
is implied by the condition in Theorem~\ref{thm:well-defined-likelihood}:
\[
\sup\limits_{y \in Y} \int_\Theta q_\theta(y) d\theta \leq C'.
\]
The implication is shown below:
\begin{align*}
		& \int_\Theta D_\theta(A)\,d\theta 
		=
		\int_\Theta \int_A q_\theta(y) dy d\theta
		=
		\int_A \int_\Theta q_\theta(y) d\theta dy
		\\
		& 
		\qquad\qquad
		{} \leq
		\int_A \sup\limits_{y' \in Y}\left(\int_\Theta q_\theta(y') d\theta\right) dy
		\leq 
		\left(C' \cdot \int_A dy\right).
\end{align*}
Second, when $\lambda_x$ is the push-forward of the Lebesgue measure
along $\theta \longmapsto D_\theta$, our $p(y \stoch D|x)$ is the density of $\nu_x$ with respect
to $\lambda_x$. This is why we called $p(y \stoch D|x)$ conditional density.

\begin{proof} 
	The theorem claims that
	$C = \nu_x(\mathcal{D})$ is finite. But $C = \mu_x(\Theta)$ by the definition
	of the push-forward measure, and so it suffices to show the finiteness of $\mu_x(\Theta)$. Note
	\begin{equation}
		\mu_x(\Theta) = \int_\Theta \exp \left( \int_Y (\log p(y|x)) D_\theta(dy) \right) d\theta.
	\end{equation}
	We compute a finite bound of $C = \mu_x(\Theta)$ as follows:
	\begin{equation}
	\begin{aligned}
			C & \le^1 \int_\Theta \int_Y \Big(\exp(\log p(y|x))\Big)\, D_\theta(dy) d\theta \\
			& = \int_\Theta \int_Y p(y|x)\, D_\theta(dy) d\theta \\
			& {} =^2 \left(C' \cdot \int_Y  p(y|x)\,dy\right) = C' < \infty
	\end{aligned}
	\label{eqn:prop-3-proof}
	\end{equation}
	where $\le^1$ is by Jensen's inequality and $=^2$ uses the assumption of the theorem.
\end{proof}

Besides $\{\mathrm{Normal}(\theta, 1)  \mid \theta \in \mathbb{R}\}$ that we discussed already after  Theorem~\ref{thm:well-defined-likelihood},
the set  $\{\mathrm{Dirac}(\theta) \mid \theta \in \mathbb{R}\}$ satisfies the condition 
\eqref{eqn:measure-bound-requirement} in Theorem~\ref{thm:measure}. Thus,
$p(y \stoch D|x)$ can be normalized to a distribution (i.e., a probability measure) in both cases.
However, $p(y \stoch D|x)$ cannot be normalized over the space  
$\{ \beta \cdot \mathrm{Dirac}(0) + (1-\beta) \cdot \mathrm{Dirac}(\theta) \mid \theta
\in \mathbb{R}\}$ for $\beta \in (0,1)$, which consists of the mixtures of two Dirac distributions.
The condition \eqref{eqn:measure-bound-requirement}  in Theorem~\ref{thm:measure} does not hold. In fact,
if $p(0|x) > 0$, the normalization constant of $\nu_x$ in the theorem is infinite.

\section{Inference algorithms}
\label{app:algorithms}

A simple \textbf{bias-adjusted likelihood estimate} $\hat p(x, y
\stoch D)$, required for the computation of the weights in
importance sampling as well as of the acceptance ratio in
pseudo-marginal Markov chain Monte Carlo~\cite{AR09}, can be
computed based on \eqref{eqn:log-p-mc-stochastic} as
follows~\cite{CD99,NFW12,QVK+18}. Under the conditions of the
central limit theorem, the distribution of 
\[
\frac 1 N \sum_{j=1}^N \log p(x, y_j)
\]
becomes similar to the normal
distribution 
\[
\mathrm{Normal}\Big(\mu\,{=}\,\underset{y \sim D}{\mathbb{E}}[\log p(x, y)], \sigma^2\,{=}\,\frac 1 N {\underset{y \sim D}{\mathrm{\mathbb{V}ar}}[\log p(x, y)]} \Big)
\]
as $N \to \infty$. Correspondingly, the distribution of 
\[
        \exp \left(\frac 1 N \sum_{j=1}^N \log  p(x, y_j)\right)
\]
and the log-normal distribution with the same parameters become
similar under the same asymptotics. But the mean of the
log-normal distribution is $\exp(\mu + \frac {\sigma^2} 2)$.
Thus, we can construct a bias-adjusted estimate as
\begin{align}
		\label{eqn:p-x-D-estimate}
		\nonumber 
		m & {} = \frac 1 N \!\sum_{j=1}^N \log p(x, y_j),
		\\[0.5ex]
		\nonumber
		s^2 & {} = \frac 1 {N{-}1} \! \sum_{j=1}^N (\log p(x, y_j){-} m)^2, 
		\\[0.5ex] 
		\hat p(x, y \stoch D) & {} = \exp(\mu)
		\\
		\nonumber & {} \approx  \mathbb{E}_{y_{1:N} \sim D^n}\!\!\left[\exp\!\left(\frac 1 N \sum\nolimits_{j=1}^N \log p(x, y_j)\right)\right] 
		\\
		\nonumber & \phantom{{}\approx{}} {} \times \exp\left(- \frac{\sigma^2}{2}\right) 
		\\
		\nonumber & \approx \exp\left(m - \frac{s^2}{2N}\right).
\end{align}

In \textbf{importance sampling}, $x_i$'s are drawn from a proposal
distribution $U$ with probability mass or density $u(x)$  and
weighted by the joint probability mass or density of $x$ and
observations.  In the case of stochastic conditioning, the
weight $w_i$ of $x_i$ is approximated as $\hat w_i$ using an
unbiased estimate $\hat p(x_i, D)$ such as
\eqref{eqn:p-x-D-estimate}.
\begin{equation}
	\hat w_{i}  = \frac{\hat p(x_i, D)}{u(x_i)} = \frac 1 {u(x_i)}\exp\left(m_i - \frac{s_i^2}{2N}\right).
	\label{eqn:w-i}
\end{equation}

\textbf{Markov chain Monte Carlo} algorithms are broadly applied to inference
in probabilistic programs, with Lightweight Metropolis-Hastings~\cite{WSG11}
as the simplest and universally applicable variant. Many MCMC
variants involve proposing a new state $x'$ from a proposal
distribution $U$ with probability mass or density $u(x'|x)$ and then
either accepting $x'$ or retaining $x$, with Metropolis-Hastings acceptance
ratio $\alpha$ based on the joint probability of $x'$ and observations:
\begin{equation}
	\alpha = \min \left\{1,\; \frac {u(x|x')} {u(x'|x)} \times \frac {p(x',y \stoch D)} {p(x, y \stoch D)}\right\}.
\end{equation}
Just like with importance sampling, $p(x, y \stoch D)$ cannot be computed
exactly for probabilistic programs with stochastic conditioning.
However, \citet{AR09} establish that the joint probability can
be replaced with an unbiased estimate without affecting the
stationary distribution of the Markov chain, resulting in
\textit{pseudo-marginal} MCMC. Pseudo-marginal MCMC allows
speeding up Monte Carlo inference by
subsampling~\cite{BDH17,QKV+19,DQK+19,QVK+18} and can be applied
to stochastic conditioning as well. The main challenge in
designing an efficient MCMC algorithm, for both subsampling and
stochastic conditiong, is constructing an unbiased low-variance
estimate of the joint probability. In a basic case,
\eqref{eqn:p-x-D-estimate} can be used as a bias-adjusted
estimate, resulting in the acceptance ratio $\hat \alpha$:
\begin{equation}
	\begin{aligned}
		\hat \alpha = {} 
		& \min\left\{1,\; \frac {u(x|x')} {u(x'|x)} \times \frac {\hat p(x', D)} {\hat p(x, y \stoch D)}\right\} 
		\\
		{} = {} 
		& \min\left\{1,\; \frac {u(x|x')} {u(x'|x)} \times  \exp \left( m' - m - \frac {s'^2 - s^2} {2N}\right)\right\}.
	\end{aligned}
	\label{eqn:pmmh-hat-alpha}
\end{equation}
Note that the same samples $y_1, y_2, ..., y_N$ should be used for estimating both
$m, s^2$ and $m', s'^2$~\cite{AR09}.

\textbf{Stochastic gradient Markov chain Monte Carlo}
(sgMCMC)~\citep{MCF15} can be used unmodified when the log probability is
differentiable with respect to $x$. sgMCMC uses an unbiased stochastic estimate
of the gradient of log probability density. Such estimate is
trivially obtained by drawing a single sample $y_1$ from $D$ and computing the
gradient of the log joint density of $x$ and $y$:
\begin{equation}
        \begin{aligned}
                \nabla_x \log p(x, y \stoch D) 
                & = \nabla_x \Big(\log\Big(p(x)\prod_{y \in Y} {p(y\vert x)^{q(y)dy}}\Big)\Big)
                \\
                & = \nabla_x \Big(\log\Big(\prod_{y \in Y} {p(x,y)^{q(y)dy}}\Big)\Big)
                \\
                & = \nabla_x \int_{y \in Y} q(y) \log p(x,y) dy
                \\
                & = \int_{y \in Y} q(y)\Big(\nabla_x \log p(x,y)\Big) dy
                \\
                & \approx \nabla_x \log p(x,y_1).
        \end{aligned}
\end{equation}

\textbf{Stochastic variational inference}
\citep{HBW+13,RGB14,KTR+17} requires a noisy estimate of the
gradient of the evidence lower bound (ELBO) $\mathcal{L}$. The
most basic approach is to use the score estimator that is
derived from the following equation:
\begin{equation}
	\nabla_\lambda \mathcal{L} = \mathbb{E}_{x \sim q(x\vert \lambda)} \left[(\nabla_\lambda \log q(x\vert \lambda))\left(\log \frac{p(x, y \stoch D)}{q(x\vert \lambda)}\right)\right]. 
	\label{eqn:nabla-elbo}
\end{equation}
As in the standard posterior inference setting, maximizing ELBO is equivalent to minimizing the KL divergence
from $q(x\vert \lambda)$ to $p(x\vert D)$. Substituting \eqref{eqn:prob-D-given-x0} into \eqref{eqn:nabla-elbo}, we obtain
\begin{equation}
\begin{aligned}
        \nabla_\lambda \mathcal{L} & 
                = \mathbb{E}_{x \sim q(x\vert \lambda)}\bigg[\nabla_\lambda \log q(x\vert \lambda)\Big(\log p(x) + {} \\
		& \hspace{4.5em}\int\limits_{y \in Y} \!\!q(y) \log p(y\vert x)dy - \log q(x\vert \lambda)\Big)\bigg]
        \\
        & 
        = \mathbb{E}_{x \sim q(x\vert \lambda)}\bigg[\int\limits_{y \in Y}\!\! \nabla_\lambda \log q(x\vert \lambda)\Big(\log p(x) + {} \\
		& \hspace{6.5em} \log p(y\vert x) - \log q(x\vert \lambda)\Big)q(y)dy\bigg] \\
        & 
        = \mathbb{E}_{(x,y) \sim q(x\vert \lambda) \times D}\bigg[\nabla_\lambda \log q(x\vert \lambda)\Big(\log p(x) + {} \\
		& \hspace{10.5em} \log p(y\vert x) - \log q(x\vert \lambda)\Big)\bigg]. 
\end{aligned}
	\label{eqn:nabla-elbo-p-x-y}
\end{equation}
Thus, $\nabla_\lambda \mathcal{L}$ can be estimated using Monte Carlo samples $x_s, y_s \sim q(x\vert \lambda) \times D$:
\begin{equation}
	\begin{aligned}
                \nabla_\lambda \mathcal{L} &\approx \frac 1 S \sum_{s=1}^S \nabla_\lambda \log q(x_s\vert \lambda)\Big(\log p(x_s) + {} \\
		& \hspace{6em} \log p(y_s\vert x_s) - \log q(x_s\vert \lambda)\Big),
	\end{aligned}
	\label{eqn:nabla-elbo-p-x-y-s-s}
\end{equation}
and stochastic variational inference can be
directly applied. In fact, \citet{MPT+16} use black-box
variational inference~\citep{RGB14} for a special case of stochastic
conditioning arising in policy search.
\end{document}